%% file: arxiv.tex
\documentclass{article}
\usepackage{amsmath,amsfonts,amssymb,amsthm,bbm,booktabs, setspace}
\usepackage[normalem]{ulem}
\usepackage{algorithm} 
\usepackage{algpseudocode} 
\usepackage{mathtools}
\usepackage{commath}
\usepackage{graphicx}
\usepackage{xcolor}
\usepackage{setspace}
\usepackage{hyperref}
\usepackage[margin = 1 in]{geometry}
\newcommand\mc{\mathcal}

\newtheorem{definition}{Definition}[section]
\newtheorem{theorem}{Theorem}[section]

\newtheorem{lemma}{Lemma}[section]
\newtheorem{example}{Example}[section]

\newtheorem{problem}{Problem}[section]
\newtheorem{assumption}{Assumption}[section]

\newcommand{\supp}{\mathrm{supp}}

\newcommand{\sign}{\mathrm{sign}}

\title{Efficient Support Recovery of Mixtures of Sparse Linear Classifiers with Fewer Measurements}
\author{Xiaxin Li \and Arya Mazumdar
\thanks{The authors are with the Hal{\i}c{\i}o\u{g}lu Data Science Institute, University of California, San Diego. emails: \texttt{\{xil095, arya\}@ucsd.edu}.
}
}
\date{}
\begin{document}

\maketitle
\input{0_abstract}

\input{1_introduction}

%\input{2_problemdescription}

\input{4_twostage}

\input{5_multistage}
\input{6_nonadaptive}
\input{7_conclusion}

%\textcolor{red}{Acknowledgements TBD}

\bibliographystyle{alpha}
\bibliography{references}

\input{Appendices}

\end{document}

%% file: 0_abstract.tex
\begin{abstract}
The support recovery problem in mixture of linear classifiers aims to identify 
% \sout{which features actually matter} 
the features relevant to the underlying decision rules when data is generated by a mixture of several linear decision rules. In particular, the goal is to recover the support (nonzero coordinates) of $l$ unknown $k$-sparse vectors from sign measurements. Each measurement is generated by selecting one of the $l$ vectors uniformly at random, and returning the sign of its inner product with a chosen measurement vector.

In this paper, we propose adaptive and non-adaptive schemes that significantly improve upon prior results by simultaneously reducing the number of measurements and achieving sublinear decoding time. In particular, our adaptive constructions substantially reduce measurements compared to existing approaches, while also lowering decoding complexity from super-quadratic to sublinear in the ambient dimension. We further provide a non-adaptive scheme that improves previous measurement bounds while maintaining efficient decoding.

Overall, our approach yields a more efficient trade-off between sample complexity and decoding time for support recovery in mixture models compared to previously known methods.

\end{abstract}

%% file: 1_introduction.tex
\section{Introduction}\label{sec:intro}
Learning a halfspace (linear binary classifier) is a fundamental problem in machine learning, where the goal is to identify a hyperplane that separates data points according to their labels. A natural extension imposes a sparsity constraint on the classifier by assuming that the normal vector in $\mathbb{R}^n$ has support of size at most $k\ll n$. This formulation is equivalent to the problem of $1$-bit compressed sensing (1bCS)~\cite{boufounos20081,acharya2017improved,matsumoto2023improved,li2026support}, which seeks to recover a sparse signal from quantized (sign) measurements. A further extension considers a mixture of $l$ sparse classifiers. In this setting, each label is generated by first selecting one of the $l$ classifiers uniformly at random and then observing on which side of the selected hyperplane the feature vector lies~\cite{sun2014learning,yin2018learning}. This problem, known as the recovery of mixtures of sparse linear classifiers (MLC), has attracted considerable attention~\cite{gandikota2020recovery,pal2021support,polyanskii2021learning,chen2022algorithms,mazumdar2024support}. Existing approaches typically recover the supports of the classifiers first and then estimate the corresponding coefficients. Since the first stage dominates the measurement complexity~\cite{gandikota2020recovery}, we focus on the support recovery problem. 

In this paper, we propose new support recovery schemes for the MLC model with significantly improved efficiency. Our contributions are twofold. First, we substantially reduce the decoding complexity from the $\Omega(n^2)$ required by existing methods to sublinear in $n$, by leveraging recent advances in group testing and one-bit compressed sensing~\cite{cai2017efficient,price2020fast,lee2019saffron,price2023fast,li2025noisy,li2026support}. Second, we also reduce the number of measurements. By allowing a small amount of adaptivity, our schemes improve the measurement complexity by more than a factor of $k$. Even in the non-adaptive setting, we achieve nearly a factor of $l$ improvement over prior work~\cite{gandikota2020recovery,pal2021support}.

\subsection{Formal problem: Support recovery of mixtures of $k$-sparse linear classifiers}\label{sec:MLCk}

We begin with some notation. Let $[n]:=\{1,\dots,n\}$ for $n\in\mathbb{N}^+$, and let $\lceil r\rceil$ denote the ceiling of $r\in\mathbb{R}$. For $v\in\mathbb{R}^n$ (or $\{0,1\}^n$), define
\(
\supp(v):=\{i\in[n]:v_i\neq0\}.
\)

Let $\sign(a)=1$ if $a\ge0$ and $\sign(a)=-1$ otherwise. We write $\langle\cdot,\cdot\rangle$ for the inner product, $|S|$ for the cardinality of a set $S$, and $\varnothing$ for the empty set.

Let
\[
B:=\{\beta_i\in\mathbb{R}^n:i\in[l]\}
\]
be a collection of $l$ unknown $k$-sparse vectors, called \emph{candidate vectors}, where $|\supp(\beta_i)|\le k$ for every $i\in[l]$ and $k\ll n$. The goal is to recover the supports of all candidate vectors from a sequence of carefully designed \emph{measurement vectors}
\[
V:=\{v_j\in\mathbb{R}^n:j\in[m]\}.
\]
For each query $j\in[m]$, an oracle $\mathcal{O}$ independently selects an index $i\in[l]$ uniformly at random and returns
\[
y_j=\sign\bigl(\langle v_j,\beta_i\rangle\bigr).
\]
The responses form the \emph{result vector}
\[
y=[y_1,\dots,y_m]^T.
\]
Some of our schemes are \emph{adaptive}, meaning that the measurement vectors are chosen in multiple stages, where the measurements in each stage may depend on the outcomes of the previous stages. We now formally state the problem.

\begin{problem}\label{Problem1}
(Support recovery for mixtures of sparse classifiers, cf.~\cite{gandikota2020recovery})
Given query access to the oracle $\mathcal{O}$, design a (possibly adaptive) measurement and decoding scheme that outputs vectors $\{\hat{\beta}_1,\dots,\hat{\beta}_l\}$ such that there exists a permutation $\sigma:[l]\to[l]$ satisfying
\[
\supp(\hat{\beta}_{\sigma(i)})=\supp(\beta_i),
\qquad \forall\, i\in[l].
\]
\end{problem}

We represent the measurement vectors by a matrix $A\in\mathbb{R}^{m\times n}$ whose $j$-th row is $A^j=v_j^T$, allowing the measurement process to be written compactly in terms of $A$. We also impose the following standard identifiability assumption.

\begin{assumption}\label{Assumption1}
(Assumption~1 in~\cite{gandikota2020recovery})
For every $i\in[l]$,
\[
\supp(\beta_i)\nsubseteq\bigcup_{j\neq i}\supp(\beta_j),
\]
i.e., each candidate vector contains at least one support element that does not appear in any other candidate vector.
\end{assumption}

\subsection{Contributions and techniques}

We study the support recovery problem for mixtures of $l$ $k$-sparse, $n$-dimensional linear classifiers under the assumption that each support contains at least one coordinate unique to that classifier. We develop three recovery schemes with different trade-offs between adaptivity, measurement complexity, and decoding complexity.

We first propose a {\em two-stage adaptive scheme} with
% \[
% m=O\!\left(k^2l^{4+\epsilon}\log^2(n)\log(kl\log(n)/\delta)\right)
% \]

\[
m=O\!\left(k^2l^4\log^2(n)\log(kl\log(n)/\delta)\right)
\]
measurements and decoding complexity $O(km)$, where 
% \sout{$\epsilon=0.001$ can be made arbitrarily small and }
$\delta=1/\lambda>0$ is the failure probability (Section~\ref{sec:2SUMLC}, Theorem~\ref{thm1}). The first stage reduces the problem to recovering the union support using the EDOCS-AE framework of~\cite{li2026support}, a universal one-bit compressed sensing scheme with fast decoding. The second stage then reconstructs each candidate vector by analyzing the frequencies of single coordinates and coordinate pairs within the recovered union support.

We next develop a {\em multi-stage adaptive scheme} with
% \[
% m=O\!\left(kl^{4+\epsilon}\log(kl\lambda)\log(n)\log(kl\lambda\log n)\right)
% \]
\[
m=O\!\left(kl^4\log(kl\lambda)\log(n)\log(kl\lambda\log n)\right)
\]
measurements and decoding complexity $O(m)$ (Section~\ref{sec:3SPMLC}, Theorem~\ref{thm2}). The improved measurement complexity is achieved by replacing EDOCS-AE with the probabilistic EDOCS-EE scheme~\cite{li2026support} and by introducing a small number of additional adaptive stages. These stages enable the efficient clustering algorithm of~\cite{black2025learning}, which reduces the number of pairwise queries from $O(n^2)$ to $O(nk)$ while requiring only $O(\log\log n)$ adaptive stages (Here, $n$ and $k$ refer to the notations used in~\cite{black2025learning}, which have meanings analogous to our corresponding notations).

Finally, we develop a {\em non-adaptive scheme} with
% \[
% m=O\!\left(k^3l^{5+\epsilon}\log^2(n)\log(kl\lambda\log n)\right)
% \]

\[
m=O\!\left(k^3l^5\log^2(n)\log(kl\lambda\log n)\right)
\]
measurements and decoding complexity $O(klm)$ (Section~\ref{sec:1SMLC}, Theorem~\ref{thm3}). The construction uses $(d,r,z]$-disjunct matrices~\cite{chen2008upper} to perform the same frequency-based analysis non-adaptively. Compared with the best previous result~\cite{gandikota2020recovery}, it reduces the number of measurements by nearly a factor of $l$ while achieving decoding complexity sublinear in $n$.

\input{3_relatedworks}

% \xl{
% \paragraph{Useful Notations.}
% For $A \in \mathbb{R}^{m \times n}$, let $A^i$ and $A_j$ denote its $i$-th row and $j$-th column, and $A_{ij}$ its entries. For $I \subseteq [n]$, let $\mathrm{ind2vec}(I) \in \{0,1\}^n$ be its indicator vector.
% Let $I_n$ be the $n \times n$ identity and $0^n$ the all-zero vector. 

% For $S=\{s_1,...,s_{|S|}\} \subseteq [m]$, $T=\{t_1,...,t_{|T|}\} \subseteq [n]$, let
% \[
% A^{S} := [A^{s_1};\dots;A^{s_{|S|}}], 
% A_{T} := [A_{t_1},\dots,A_{t_{|T|}}].
% \]

% For $A \in \mathbb{R}^{m \times n}$, define $\mathrm{Rep}(A,R)\in\mathbb{R}^{mR\times n}$ by repeating each row of $A$ exactly $R$ times:
% \[
% \mathrm{Rep}(A, R)
% :=
% [\underbrace{A^1; \dots; A^1}_{R\ \text{times}};
% \underbrace{A^2; \dots; A^2}_{R\ \text{times}};
% \dots;
% \underbrace{A^m; \dots; A^m}_{R\ \text{times}}].
% \]
% For $V=\{v^{(i)}\}_{i=1}^l \subset \mathbb{R}^n$, define
% \[
% \mathrm{Freq}(j,V)=|\{i : j \in \supp(v^{(i)})\}|,\quad
% \mathrm{Freq}((j_1,j_2),V)=|\{i : \{j_1,j_2\}\subseteq \supp(v^{(i)})\}|.
% \]

% The full proof of lemmas and theorems can be found in the appendix.
% }

\paragraph{Useful Notations.}
For $A\in\mathbb{R}^{m\times n}$, let $A^i$, $A_j$, and $A_{ij}$ denote its $i$-th row, $j$-th column, and $(i,j)$-th entry, respectively. For $I\subseteq[n]$, let $\mathrm{ind2vec}(I)\in\{0,1\}^n$ denote its indicator vector. Let $I_n$ and $0^n$ denote the identity matrix and the all-zero vector, respectively.

For $S=\{s_1,\dots,s_{|S|}\}\subseteq[m]$ and $T=\{t_1,\dots,t_{|T|}\}\subseteq[n]$, define
\[
A^S:=[A^{s_1};\dots;A^{s_{|S|}}],
\qquad
A_T:=[A_{t_1},\dots,A_{t_{|T|}}].
\]

For $A\in\mathbb{R}^{m\times n}$, let $\mathrm{Rep}(A,R)\in\mathbb{R}^{mR\times n}$ denote the matrix obtained by repeating every row of $A$ exactly $R$ times:
\[
\mathrm{Rep}(A,R)
:=
[\underbrace{A^1;\dots;A^1}_{R\text{ times}};
\underbrace{A^2;\dots;A^2}_{R\text{ times}};
\dots;
\underbrace{A^m;\dots;A^m}_{R\text{ times}}].
\]
For $V=\{v^{(i)}\}_{i=1}^l\subset\mathbb{R}^n$, define
\[
\mathrm{Freq}(j,V):=
|\{i:j\in\supp(v^{(i)})\}|,
\qquad
\mathrm{Freq}((j_1,j_2),V):=
|\{i:\{j_1,j_2\}\subseteq\supp(v^{(i)})\}|.
\]
Complete proofs of the  lemmas and theorems can be found in the appendix.

%% file: 3_relatedworks.tex
\paragraph{Comparison with the MLC scheme in \cite{gandikota2020recovery}.}
The work most closely related to ours is~\cite{gandikota2020recovery}, which provides the previous best support recovery result for Problem~\ref{Problem1} under Assumption~\ref{Assumption1}. Their non-adaptive scheme uses
\[
m=O(k^3l^6\log(n)\log(kln))
\]
measurements, has decoding complexity $\Omega(n^2)$, and succeeds with failure probability $O(n^{-2})$ (Theorem~1 therein). Their decoding strategy estimates quantities of the form $|\mc{S}(i)|$ and $|\mc{S}(i)\cup\mc{S}(j)|$ for all coordinate pairs $(i,j)\in[n]^2$, where $\mc{S}(i)$ denotes the set of candidate vectors containing coordinate $i$. Recovering these quantities requires the stated number of measurements, while the decoding complexity is dominated by computing all entries of the $n\times n$ matrix $XX^T$, where $X\in\mathbb{R}^{n\times l}$ represents the candidate vectors in matrix form.

In comparison, all of our schemes achieve polynomial improvements in $k$ and $l$ in measurement complexity, while reducing the decoding complexity to $\operatorname{poly}(k,l,\log n)$. This improvement is enabled by two key ideas. First, instead of relying on combinatorial constructions such as RUFF~\cite{acharya2017improved}, we use one-bit compressed sensing matrices from~\cite{li2026support} as fundamental building blocks. These measurement matrices are associated with fast decoding algorithms and are particularly compatible with the MLC setting, whereas directly applying other 1bCS constructions~\cite{acharya2017improved,matsumoto2023improved,matsumoto2024robust} can be challenging.

Second, our schemes avoid processing irrelevant coordinates. Since the union of all supports contains at most $kl$ coordinates, the decoding procedure only operates on this reduced set, whereas the approach of~\cite{gandikota2020recovery} considers all $n$ coordinates, leading to quadratic decoding complexity in terms of $n$.

%% file: 4_twostage.tex
% \xl{\section{MLC: Two-stage scheme}\label{sec:2SUMLC}
% In this section, we present our two-stage scheme for solving Problem~\ref{Problem1} under Assumption~\ref{Assumption1}. At a high level, the first stage recovers the support of the union of all candidate vectors (the \emph{union support}), while the second stage identifies each candidate vector individually by analyzing the frequencies of single coordinates and pairs of coordinates within the recovered union support.
% }

\section{MLC: Two-stage scheme}\label{sec:2SUMLC} 
In this section, we present our two-stage scheme for solving Problem~\ref{Problem1} under Assumption~\ref{Assumption1}. Define the \emph{union support}
\[
\beta:=\bigcup_{i=1}^{l}\supp(\beta_i),
\]
as the union of the supports of all candidate vectors. At a high level, the first stage recovers $\beta$, while the second stage identifies individual candidate vectors by analyzing the frequencies of single coordinates and coordinate pairs within $\beta$.

% For the first stage, we use the EDOCS-AE measurement matrix $A$ from~\cite{li2026support} with sparsity parameter $kl$. This measurement matrix originates from~\cite[Theorem~2.3.2]{li2026support}, claiming that there exists a matrix $A\in\mathbb{R}^{m'\times n}$ with
% \[
% m'=O(k^2l^2\log^2 n)
% \]
% measurements and a recovery algorithm with runtime $D=O(kl\,m')$ that universally recovers the support of any $kl$-sparse real-valued signal.

For the first stage, we use the EDOCS-AE measurement matrix $A$ from~\cite{li2026support} with sparsity parameter $kl$. Specifically,~\cite[Theorem~2]{li2026support} establishes the existence of a matrix $A\in\mathbb{R}^{m'\times n}$ with

\[
m'=O(k^2l^2\log^2 n)
\]
measurements, together with a recovery algorithm with runtime $D=O(klm')$ that universally recovers the support of any $kl$-sparse real-valued signal.

Here, a key difference from standard one-bit compressed sensing is that measurements are not obtained from a fixed sparse vector. Instead, each measurement independently samples one of the $l$ candidate vectors uniformly at random, so repeated measurements of the same vector may correspond to different candidates. To overcome this randomness, we repeat each measurement vector of $A$ sufficiently many times so that every candidate vector is sampled at least once with high probability.

For a fixed measurement vector $A^j$, we define the \emph{AC event} (All Covered) as the event that every pair $(A^j,\beta_i)$, $i\in[l]$, appears at least once among the repetitions. The \emph{global AC event} is the event that the AC event holds simultaneously for all $j\in[m']$.

To exploit this event, we repeat each measurement vector $A^j$ and its negation $-A^j$ exactly $R$ times. Suppose that the AC event holds for both $A^j$ and $-A^j$. If
\(
\supp(A^j)\cap\beta\neq\varnothing,
\)
then there exists some $\beta_i$ satisfying
\(
\supp(A^j)\cap\supp(\beta_i)\neq\varnothing,
\)
and hence $\langle A^j,\beta_i\rangle\neq0$ (provided that the accidental zero issue does not arise; see below for its definition). Therefore, both
\(
\sign(\langle A^j,\beta_i\rangle) \text{ and }
\sign(\langle -A^j,\beta_i\rangle)
\)
appear among the repeated measurements, and one of them must equal $-1$. Thus, after aggregation, each measurement behaves like a group test: all $1$'s indicate that the measurement is disjoint from the union support, while the presence of a $-1$ indicates an intersection.

The following lemma gives the number of repetitions required to guarantee the global AC event with high probability.

\begin{lemma}\label{lem:rep1}
Suppose each measurement vector $A^j$ and its negation $-A^j$ are repeated
\[
R:=\left\lceil l\log(2m'l\lambda)\right\rceil
\]
times. Then the global AC event occurs with probability at least $1-1/\lambda$. Equivalently, for every $j\in[m']$ and $i\in[l]$, both
\[
\sign(\langle A^j,\beta_i\rangle)
\quad\text{and}\quad
\sign(\langle -A^j,\beta_i\rangle)
\]
are observed at least once.
\end{lemma}

Accordingly, we define the first-stage measurement matrix as
\(
A^{(1)}:=[\mathrm{Rep}(A,R);\mathrm{Rep}(-A,R)].
\)

The resulting measurement vector is partitioned as
\(
y=[y_1^1;y_2^1;\cdots;y_{m'}^1;
   y_1^2;y_2^2;\cdots;y_{m'}^2],
\)
where each $y_i^j\in\mathbb{R}^R$ for $i\in[m']$ and $j\in\{1,2\}$. 

We construct the \emph{tentative result vector}
\(
y^0=[y_1^0;y_2^0;\cdots;y_{m'}^0],
\)
with
\(
y_i^0=
\begin{cases}
0,&\text{if }y_i^1=y_i^2=1^R,\\
1,&\text{otherwise.}
\end{cases}
\)

The union support $\beta$ is then recovered from $y^0$ and $A$ using the decoder from~\cite[Algorithm~3]{li2026support}.

A subtle issue is the possibility of an \emph{accidental zero}~\cite{li2026support}: a measurement vector $v$ may satisfy
\[
\supp(v)\cap\beta\neq\varnothing,
\]
while still having
\[
\langle v,\beta_i\rangle=0
\]
for every $\beta_i\in B$. In this case, all repetitions of both $v$ and $-v$ return $1$, causing the corresponding entry of $y^0$ to be incorrectly set to zero. The following lemma shows that such accidental zeros do not affect the correctness of the union support recovery.

\begin{lemma}\label{lem:trueb}
    Given the success of the global AC event, by using $y^0$ and $A$, applying \cite[Algorithm~3]{li2026support} will yield the true union support $\beta$. 
\end{lemma}

% \xl{Having obtained the union support $\beta$, we proceed to the second stage of our scheme. The main idea is to recover each $\supp(\beta_i)$ by examining the \emph{frequency} of individual coordinates in $\beta$ and the frequency of coordinate pairs within $\beta$ across the set $B$. In general, these two types of frequencies are not sufficient to recover $B$ completely. However, under Assumption~\ref{Assumption1}, this approach becomes effective: We first identify all \emph{singletons} as the set $S$, i.e., coordinates that appear in exactly one candidate vector. (These singletons are analogous in spirit, but not identical, to the notion of singletons in \cite{li2026support}.) This is done by repeatedly querying all $b \in \beta$ and selecting coordinates that appear with probability approximately $1/l$. Next, we query all pairs of coordinates within $S$ to determine how singletons are distributed across $B$. Finally, for each candidate vector, we query pairs $(s,t)$ with $s \in S$ and $t \in \beta \setminus S$ to recover the remaining coordinates. Notably, it suffices to select one singleton from each candidate vector, rather than iterating over all singletons.}

Having recovered the union support $\beta$, we proceed to the second stage of our scheme, whose goal is to recover each individual support $\supp(\beta_i)$. The key idea is to exploit the \emph{frequencies} of individual coordinates and coordinate pairs within $\beta$ across the candidate vectors. While these frequency statistics alone are generally insufficient to uniquely recover $B$, Assumption~\ref{Assumption1} guarantees the existence of coordinates that uniquely identify each candidate vector. 

Specifically, we first identify the set $S$ of all \emph{singletons}, namely, coordinates that appear in exactly one candidate vector. (This notion is analogous in spirit, but not identical, to the singletons in~\cite{li2026support}.) We obtain $S$ by repeatedly querying each coordinate in $\beta$ and selecting those whose frequency is approximately $1/l$. Next, we examine all pairs of singleton coordinates to determine their assignment among the candidate vectors. Finally, for each partially recovered candidate vector, we query pairs $(s,t)$ with $s\in S$ and $t\in\beta\setminus S$ to identify its remaining coordinates. Note that for each candidate vector, only one singleton representative is needed in this process. 

% \xl{
% \begin{example}
% Let $l=3$, $n=10$, and
% \[
% \beta_1 = [0;1;0;0;0;0;1;1;0;0],\quad
% \beta_2 = [0;0;1;0;0;0;0;1;1;0],\quad
% \beta_3 = [1;0;0;1;1;0;1;1;1;0].
% \] 
% In the decoding algorithm, we first recover the union support 
% $
% \beta = \{1,2,3,4,5,7,8,9\},
% $ 
% and then identify the singletons
% $
% S = \{1,2,3,4,5\}.
% $ 
% Next, we query all pairs of coordinates in $S$. Observing that 
% $
% \text{Freq}((1,4),B) = \text{Freq}((1,5),B) = \text{Freq}((4,5),B) = 1
% $ 
% and all other pairs have zero frequency, we deduce the singleton distribution as $\{2\}, \{3\}, \{1,4,5\}$. Accordingly, we set 
% \[
% \beta_1 = [0;1;0;0;0;0;?;?;?;0],\quad
% \beta_2 = [0;0;1;0;0;0;?;?;?;0],\quad
% \beta_3 = [1;0;0;1;1;0;?;?;?;0].
% \] 
% Finally, to recover the remaining unknown coordinates, we query $(s,t)$ for $s \in \{1,2,3\}$ and $t \in \{7,8,9\}$, yielding the complete supports of $\beta_1$, $\beta_2$, and $\beta_3$. (Only one singleton from $\{1,4,5\}$ needs to be chosen.)
% \end{example}
% }

\begin{example}
Let $l=3$, $n=10$, and
\[
\beta_1=[0;1;0;0;0;0;1;1;0;0],\quad
\beta_2=[0;0;1;0;0;0;0;1;1;0],\quad
\beta_3=[1;0;0;1;1;0;1;1;1;0].
\]
The first stage recovers the union support
\(
\beta=\{1,2,3,4,5,7,8,9\}.
\)

In the second stage, we first identify the singleton set
\(
S=\{1,2,3,4,5\}.
\)

We then examine all pairs of coordinates in $S$. Since
\(
\mathrm{Freq}((1,4),B)
=\mathrm{Freq}((1,5),B)
=\mathrm{Freq}((4,5),B)=1,
\)
while every other pair of singletons has frequency zero, we recover the singleton partition
\(
\{1,4,5\},\{2\},\{3\}.
\)

Therefore, up to a permutation of the candidate vectors, we obtain
\[
\beta_1=[0;1;0;0;0;0;?;?;?;0],\quad
\beta_2=[0;0;1;0;0;0;?;?;?;0],\quad
\beta_3=[1;0;0;1;1;0;?;?;?;0].
\]
Finally, we query pairs $(s,t)$ with $s\in\{1,2,3\}$ and $t\in\{7,8,9\}$ to recover the remaining coordinates and hence the complete supports. Note that only one representative singleton from the group $\{1,4,5\}$ is needed for this final step (which is $1$ in this case).
\end{example}

The remaining question is then to find these frequencies with high success probability, and the idea is again to repeat measurements, as in stage one, but with a different purpose (Recall that the repetition parameter $R$ in stage one aims for every measurement vector to have a chance to measure all candidate vectors). Since we are counting singletons and pairs, we will define the following matrices and use them building blocks of our measurements:
\begin{definition}\label{def:nbspmatrix}
    The $(n,\beta)$-singletons matrix $M^S_{n,\beta}\in\{0,1\}^{|\beta|\times n}$ is defined such that $(M^S_{n,\beta})_\beta=I_{|\beta|}$, and the $(n,\beta)$-pairs matrix $M^P_{n,\beta}\in\{0,1\}^{\frac{|\beta|^2-|\beta|}{2}\times n}$ is defined such that for each row of $(M^P_{n,\beta})_\beta$, a pair of entries is 1 and other entries are zero, the 1 entries of any row follows by the dictionary order from the previous row. The $(n,\beta)$-singletons and pairs matrix $M_{n,\beta}\in\{0,1\}^{\frac{|\beta|^2+|\beta|}{2}\times n}$ is then defined as the row-wise concatenation of $M^S_{n,\beta}$ and $M^P_{n,\beta}$. 
\end{definition}
% \begin{remark}\label{rmk:matrix}
%   To avoid the accidental zero issue, we modify the matrices $M_{n,\beta}^P$ and $M_{n,\beta} $ above by replacing each row with two 1s by 2 rows with 1,1 and 1,2 in corresponding locations, and we denote these matrices as $M_{n,\beta}^{P'}\in\mathbb{R}^{(|\beta|^2-|\beta|)\times n}$ and $M'_{n,\beta}\in\mathbb{R}^{|\beta|^2\times n}$, respectively. The reason behind this modification is that if $(x_1,x_2)\neq(0,0)$, then at least one of $x_1+x_2$ and $x_1+2x_2$ is nonzero. In all subsequent analysis, we refer to the matrices $M_{n,\beta}^P$ and $M_{n,\beta}$ with no accidental zero issue. 
% \end{remark}

\begin{example}
   Let $n=8$, $\beta=\{3,4,5\}$, then 
   \[M_{n,\beta}^S=
   {
   \renewcommand{\arraystretch}{0.9} 
   \setlength{\arraycolsep}{3pt}
   \begin{bmatrix}
   0 & 0 & 1 & 0 & 0 & 0 & 0 & 0  \\
   0 & 0 & 0 & 1 & 0 & 0 & 0 & 0  \\
   0 & 0 & 0 & 0 & 1 & 0 & 0 & 0  \\
   \end{bmatrix} 
   }
   , M_{n,\beta}^P=
   {
   \renewcommand{\arraystretch}{0.9} 
   \setlength{\arraycolsep}{3pt}\begin{bmatrix}
   0 & 0 & 1 & 1 & 0 & 0 & 0 & 0 \\
   0 & 0 & 1 & 0 & 1 & 0 & 0 & 0 \\
   0 & 0 & 0 & 1 & 1 & 0 & 0 & 0 
   \end{bmatrix}
   }, M_{n,\beta}=
   {
   \renewcommand{\arraystretch}{0.9}
   \setlength{\arraycolsep}{3pt}
   \begin{bmatrix}
   0 & 0 & 1 & 0 & 0 & 0 & 0 & 0 \\
   0 & 0 & 0 & 1 & 0 & 0 & 0 & 0 \\
   0 & 0 & 0 & 0 & 1 & 0 & 0 & 0 \\
   0 & 0 & 1 & 1 & 0 & 0 & 0 & 0 \\
   0 & 0 & 1 & 0 & 1 & 0 & 0 & 0 \\
   0 & 0 & 0 & 1 & 1 & 0 & 0 & 0 
   \end{bmatrix}
   }.
   % M_{n,\beta}^{P'}=
   % {
   % \renewcommand{\arraystretch}{0.9}
   % \setlength{\arraycolsep}{3pt}
   % \begin{bmatrix}
   %  0 & 0 & 1 & 1 & 0 & 0 & 0 & 0 \\
   % 0 & 0 & 1 & 2 & 0 & 0 & 0 & 0 \\
   % 0 & 0 & 1 & 0 & 1 & 0 & 0 & 0 \\
   % 0 & 0 & 1 & 0 & 2 & 0 & 0 & 0 \\
   % 0 & 0 & 0 & 1 & 1 & 0 & 0 & 0 \\
   % 0 & 0 & 0 & 1 & 2 & 0 & 0 & 0 \\
   % \end{bmatrix}
   % },M'_{n,\beta}=
   % {
   % \renewcommand{\arraystretch}{0.9}
   % \setlength{\arraycolsep}{3pt}
   % \begin{bmatrix}
   % 0 & 0 & 1 & 0 & 0 & 0 & 0 & 0 \\
   % 0 & 0 & 0 & 1 & 0 & 0 & 0 & 0 \\
   % 0 & 0 & 0 & 0 & 1 & 0 & 0 & 0 \\
   % 0 & 0 & 1 & 1 & 0 & 0 & 0 & 0 \\
   % 0 & 0 & 1 & 2 & 0 & 0 & 0 & 0 \\
   % 0 & 0 & 1 & 0 & 1 & 0 & 0 & 0 \\
   % 0 & 0 & 1 & 0 & 2 & 0 & 0 & 0 \\
   % 0 & 0 & 0 & 1 & 1 & 0 & 0 & 0 \\
   % 0 & 0 & 0 & 1 & 2 & 0 & 0 & 0 \\
   % \end{bmatrix}}.
   \]
\end{example}

The second-stage measurement matrix is obtained by repeating both $M_{n,\beta}$ and its negation $R'$ times, namely,
\[
A^{(2)} := [\mathrm{Rep}(M_{n,\beta}, R');\mathrm{Rep}(-M_{n,\beta}, R')].
\]
% Since the sampling process selects each $\beta_i$ uniformly with probability $1/l$, the expected number of positive measurements for a coordinate $j$ with frequency $f_j$ in $B$ is 
% $
% \frac{f_j}{l} \cdot R' = \frac{R' f_j}{l}.
% $ 
% In particular, for a singleton coordinate $j_s$ (with $f_{j_s} = 1$), we expect only $R'/l$ positive measurements in total.
In the first part of the second stage, our goal is to distinguish singleton coordinates from all others. To do this, we will find the exact frequencies of all coordinates in $\beta$ and select those with frequency equal to one. 
% A natural criterion is to set the upper threshold for the number of positive measurements among the $R'$ repetitions of a singleton coordinate to be 
% $
% \frac{3R'}{2l}.
% $ 
Next, we find the frequencies of all pairs $(i_1, i_2)$ where both $i_1$ and $i_2$ are singletons. 
% By definition, $\text{Freq}((i_1,i_2),B)$ can only be 0 or 1. If $\text{Freq}((i_1,i_2),B) = 0$, then $i_1$ and $i_2$ belong to different candidate vectors, and we expect a fraction of $2/l$ of the corresponding queries to be positive. Conversely, if $\text{Freq}((i_1,i_2),B) = 1$, then $i_1$ and $i_2$ belong to the same candidate vector, and the expected fraction of positive queries is $1/l$.
Finally, we find the frequency of pairs $(i_1,i_2)$ where $i_1$ is a singleton and $i_2$ is not. By Assumption~\ref{Assumption1}, the frequency $\text{Freq}((i_1,i_2),B)$ is again either 0 or 1.
% so the expected number of positive measurements is analyzed in the same manner as above.

The following lemma formalizes the choice of $R'$ of finding all the exact frequencies: 

\begin{lemma}\label{lem:rep2}
    Let 
    % \sout{$R':=\lceil l^{2+\epsilon}\cdot r'\rceil$} 
    $R':=\lceil l^2\cdot r'\rceil$ where 
    %\sout{$\epsilon=0.001$ and $r'=14\log(2k^2l^2\lambda)$}
    $r'=2\log(4k^2l^2\lambda)$ s.t. $\lambda>0$. Define $X^{(i)}:=X_1^{(i)}+X_2^{(i)}+...+X_{R'}^{(i)}$, where each $X_j^{(i)}$ is Bernoulli $(\frac{f^{(i)}}{l})$ for all $i\in[2|\beta|^2]$ and $f^{(i)}\in[l]$. Then the desired event that for all $i$, $\frac{(f^{(i)}-0.5)R'}{l}<X^{(i)}<\frac{(f^{(i)}+0.5)R'}{l}$ fails with probability $<\frac{1}{\lambda}$. 
\end{lemma}
We present our construction of the measurement vectors and the decoding algorithm regarding our two-stage scheme below (Algorithm~\ref{alg:2SUMLC},~\ref{alg:mlcdadec}), and we conclude this section with the following theorem: 
\begin{theorem}\label{thm1}
    Let $B:=\{\beta_1,\beta_2,...,\beta_l\}$ be a set of unknown $k$-sparse vectors in $\mathbb{R}^n$ that satisfy Assumption~\ref{Assumption1}. Let 
    %\sout{$\epsilon=0.001$ and} 
    $\lambda>0$, there exists a two-stage adaptive algorithm (Algorithm~\ref{alg:2SUMLC},~\ref{alg:mlcdadec}) to recover the support of every candidate vector $\beta_i\in B$ with probability at least $1 - \frac{2}{\lambda}$, using %\sout{$m=O(k^2l^{4+\epsilon}\log^2 (n)\log (kl\lambda\log n))$}
    $m=O(k^2l^4\log^2 (n)\log (kl\lambda\log n))$ queries of the oracle $\mc{O}$, with a decoding  complexity $O(km)$.
\end{theorem}

\begin{algorithm}
% \small
\caption{Two-stage scheme of the MLC under Assumption~\ref{Assumption1}: Overall description}\label{alg:2SUMLC}
\begin{algorithmic}[1]
    \Require Dimension of candidate vectors $n$, number of candidate vectors $l$, the EDOCS-AE matrix $A\in \{0,1\}^{m'\times n}$ from \cite[Algorithm~2]{li2026support} that recovers a $kl$-sparse vector $x\in \mathbb{R}^n$, parameters $\lambda>0$, 
    % \sout{$\epsilon=0.001$},
    $R:=\lceil l\log(2m'l\lambda)\rceil$, %\sout{$R':=\lceil 14\cdot l^{2+\epsilon}\log(4k^2l^2\lambda)\rceil$} 
    $R':=\lceil 2l^2\log(4k^2l^2\lambda)\rceil$.
    \State Let $A^{(1)}:=[\mathrm{Rep}(A,R);\mathrm{Rep}(-A,R)]$ be the measurement matrix of the first stage, and denote $y\in\{-1,1\}^{2Rm'}$ as the result vector of the measurement matrix $A^{(1)}$. Partition $y=[y_1^1;y_2^1;\cdots;y_{m'}^1;
   y_1^2;y_2^2;\cdots;y_{m'}^2]$ where each $y_i^{\{1,2\}}\in \{-1,1\}^R$. Let $y^0:=[y_1^0;y_2^0;...;y_{m'}^0]\in\{0,1\}^{m'}$ s.t. $y_i^0=0$ if $y_i^1=y_i^2=1^R$, otherwise $y_i^0=1$. 
    \State Apply the decoding algorithm from \cite[Algorithm~3]{li2026support} with input matrix $A$ and input result vector $y^0$. Denote $\beta$ as the decoding result.
    \State Generate the matrix $M_{n,\beta}$ according to Definition~\ref{def:nbspmatrix}. 
    \State Let $A^{(2)}:=[\mathrm{Rep}(M_{n,\beta},R');\mathrm{Rep}(-M_{n,\beta},R')]$ be the measurement matrix of the second stage. 
    \State Denote $y^{(2)}$ as the result vector of the measurement matrix $A^{(2)}$. 
    \State Use $\beta$ and $y^{(2)}$ as input, apply Algorithm~\ref{alg:mlcdadec} below to get the final result $B$. 

\end{algorithmic}
\end{algorithm}

\begin{algorithm}
\caption{Second-stage decoding algorithm for the two-stage MLC scheme
(see Algorithm~\ref{alg:2SUMLC} for the first stage)}
\label{alg:mlcdadec}
\begin{algorithmic}[1]
\Require $n,l,\beta,y^{(2)},\lambda,R'$ from Algorithm~\ref{alg:2SUMLC}.
\State Construct the $(n,\beta)$-singletons and pairs matrix $M_{n,\beta}$ from Definition~\ref{def:nbspmatrix}.
\State Write
\(
y^{(2)}=[y^1;y^2],\) such that \(
y^j=[y_1^j;\ldots;y_{|\beta|}^j;
y_{(1,2)}^j;\ldots;y_{(|\beta|-1,|\beta|)}^j],
\)
where $j\in\{1,2\}$ and each block has length $R'$.
\State Order $\beta=\{b_1,\ldots,b_{|\beta|}\}$ with $b_1<\cdots<b_{|\beta|}$, and initialize $S=\varnothing$.
\For{$i=1,\ldots,|\beta|$}
\Comment{Identify singleton coordinates}
    \State If $y_i^1,y_i^2$ contain fewer than $\frac{1.5R'}{l}$ entries equal to $-1$ in total, add $b_i$ to $S$.
\EndFor
\State Let $S_{\rm copy}=S$, and initialize $B=\{\beta_s=e_s:s\in S\}$.
\For{each $\{\beta_{s_1},\beta_{s_2}\}\subseteq B$, $s_1<s_2$}
\Comment{Merge singleton pairs}
    \State Let $i,j$ satisfy $b_i=s_1,b_j=s_2$. If $y_{(i,j)}^1,y_{(i,j)}^2$ contain fewer than $\frac{1.5R'}{l}$ entries equal to $-1$, then set $\beta_{s_1}\leftarrow\beta_{s_1}+\beta_{s_2}$, $B\leftarrow B-\{\beta_{s_2}\}$, $S\leftarrow S-\{s_2\}$.
\EndFor
\For{$s\in S$}
\Comment{Recover non-singleton coordinates}
    \State Let $i$ satisfy $b_i=s$.
    \For{$t\in\beta-S_{\rm copy}$}
        \State Let $j$ satisfy $b_j=t$, and let $F$ satisfy
        \(
        \frac{(F-0.5)R'}{l}\leq N<\frac{(F+0.5)R'}{l},
        \)
        where $N$ is the number of $-1$ entries in $y_j^1$ and $y_j^2$ combined. If $y_{(i,j)}^1$ and $y_{(i,j)}^2$ contain fewer than
        $\frac{(F+0.5)R'}{l}$ entries equal to $-1$ in total, set
        $\beta_s\leftarrow\beta_s+e_t$.
    \EndFor
\EndFor
\State \Return $B$
\end{algorithmic}
\end{algorithm}

%% file: 5_multistage.tex
\section{MLC: Multi-stage probabilistic scheme}\label{sec:3SPMLC} 
The scheme in the previous section requires $k^2$ measurements (in terms of $k$, and up to logarithmic factors, which are omitted throughout the discussion). We next ask whether this dependence on $k$ can be improved. The answer is affirmative if we relax universal recovery and allow a small amount of adaptivity.

In this section, we present a probabilistic scheme for Problem~\ref{Problem1} using only $k$ measurements in terms of $k$, improving upon both the $k^3$-measurement scheme of~\cite{gandikota2020recovery} and the $k^2$-measurement scheme (Algorithm~\ref{alg:2SUMLC}). Unlike universal recovery, which requires deterministic success for every input, probabilistic schemes allow a small input-dependent failure probability. While universal guarantees are desirable in standard one-bit compressed sensing, our MLC setting already involves randomness in the measurement process due to the random selection of candidate vectors. Therefore, allowing a small failure probability is a natural relaxation in this setting.

The main idea is to replace the EDOCS-AE matrix with the EDOCS-EE matrix~\cite{li2026support}, a probabilistic one-bit compressed sensing measurement matrix associated with a fast decoding algorithm. Specifically,~\cite[Theorem~3]{li2026support} establishes the existence of a matrix $A\in\mathbb{R}^{m'\times n}$ with

\[
m'=O(kl\log (kl)\log n)
\]
measurements, together with a recovery algorithm with runtime $D=O(m')$ that recovers the support of any $kl$-sparse real-valued signal with high success probability. 

As a result, the number of measurements in the first stage will be reduced from $k^2$ to $k$ in terms of $k$. However, this modification alone does not reduce the overall query complexity, since the second stage of the two-stage MLC scheme still requires examining all pairs of singleton coordinates, resulting in a $k$-quadratic number of frequency checks.

% \xl{
% To overcome this bottleneck, we introduce adaptivity into the singleton clustering procedure. Let
% \[
% S=\{s_1,\ldots,s_w\}\subseteq\beta
% \]
% denote the singleton set, where $w\le kl$. There is a natural clustering process described below that will overcome this $k^2$ bottleneck: In the first adaptive stage, we compute $\mathrm{Freq}((s_1,s_j),B)$ for all $j$, thereby identifying the subset $S_1$ of coordinates that belong to the same candidate vector as $s_1$. In the next stage, we repeat this procedure on $S\setminus S_1$, and continue recursively. After $l$ adaptive stages, the entire singleton distribution is recovered. With repetition parameter $R''$ to guarantee accuracy, this procedure requires
% \[
% O(kl^2R'')
% \]
% measurements.
% }

To overcome this bottleneck, we introduce adaptivity into the singleton clustering procedure. Let
\[
S=\{s_1,\ldots,s_w\}\subseteq\beta
\]
denote the singleton set, where $w\le kl$. As a motivation, consider the following natural adaptive clustering procedure. In the first adaptive stage, we compute $\mathrm{Freq}((s_1,s_j),B)$ for all $j$, thereby identifying the subset $S_1$ of coordinates that belong to the same candidate vector as $s_1$. We then repeat this procedure on $S\setminus S_1$, and continue recursively. After $l$ adaptive stages, the entire singleton distribution can be recovered. With repetition parameter $R''$ to guarantee accuracy, this procedure requires
\[
O(kl^2R'')
\]
measurements.

Although this simple procedure eliminates the quadratic number of pairwise checks, its requirement of $l$ adaptive stages can be undesirable when $l$ is large. To further reduce the adaptivity cost, we leverage the pairwise-query algorithm of \cite[Theorem~1.1]{black2025learning}, which recovers the underlying partition using only $\lceil\log\log(kl)\rceil$ adaptive stages. We denote the query matrix used in their $j$-th stage by
\[
M^{Q^j}\in\{0,1\}^{m^{Q^j}\times n},
\qquad j\in[\lceil\log\log(kl)\rceil].
\]

\begin{example}\label{exp5.1}
    Let $n=10$ and $S=\{3,5,7,8,10\}$, and suppose in the first stage of the querying algorithm in \cite[Theorem~1.1]{black2025learning}, the queried pairs are $\{3,7\},\{5,10\},\{3,8\},\{5,7\}$, then 
    \[M^{Q^1}={
   \renewcommand{\arraystretch}{0.9}
   \begin{bmatrix}
        0 & 0 & 1 & 0 & 0 & 0 & 1 & 0 & 0 & 0 \\
        0 & 0 & 0 & 0 & 1 & 0 & 0 & 0 & 0 & 1 \\
        0 & 0 & 1 & 0 & 0 & 0 & 0 & 1 & 0 & 0 \\
        0 & 0 & 0 & 0 & 1 & 0 & 1 & 0 & 0 & 0 
    \end{bmatrix}}.
    \]
\end{example}

The frequency-checking procedure in the present probabilistic scheme is identical to that in the two-stage scheme, except for the procedure used to determine the distribution of singleton coordinates among the candidate vectors. Specifically, the previous quadratic pairwise checking procedure is replaced by the more efficient adaptive clustering procedure described above. We are therefore ready to present the measurement construction and decoding algorithm (Algorithm~\ref{alg:MSPMLC}), followed by the main theorem of this section.

\begin{theorem}\label{thm2}
    Let $B:=\{\beta_1,\beta_2,...,\beta_l\}$ be a set of unknown $k$-sparse vectors in $\mathbb{R}^n$ that satisfy Assumption~\ref{Assumption1} and $k=O(n^\alpha)$ where $0<\alpha<1$. Let 
    % \sout{$\epsilon=0.001$ and }
    $\lambda>0$, there exists a $(\lceil\log\log (kl)\rceil+3)$-stage adaptive algorithm (Algorithm~\ref{alg:MSPMLC}) to recover the support of every candidate vector $\beta_i\in B$ with probability at least $1 - (\frac{2}{\lambda}+\epsilon_1+\epsilon_2)$, where $\epsilon_1=20\log_2(n)\cdot k^{-9}$ and $\epsilon_2=e^{-k}+n^{-k}+k^{-3}$, using 
    %\sout{$m=O(kl^{4+\epsilon}\log(kl\lambda)\log (n)\log (kl\lambda\log n))$} 
    $m=O(kl^4\log(kl\lambda)\log (n)\log (kl\lambda\log n))$ queries of the oracle $\mc{O}$, with a decoding complexity $O(m)$. 
\end{theorem}

\begin{algorithm}
\caption{Multi-stage scheme of the MLC under Assumption~\ref{Assumption1}: Overall description}\label{alg:MSPMLC}
\begin{algorithmic}[1]
    \Require $n,l,\lambda,R$, $R'$ from Algorithm~\ref{alg:2SUMLC}. EDOCS-EE matrix $A^E\in \mathbb{R}^{m^E\times n}$ from \cite[Algorithm~4]{li2026support} that recovers a $kl$-sparse vector $x\in \mathbb{R}^n$. the matrix $M_{n,\beta}^S$ according to Definition~\ref{def:nbspmatrix}.
    
    % \begin{scriptsize}
    % \Comment{In the first stage we find the union support $\beta$}    
    % \end{scriptsize} 
    \State Construct
    \(
    A^{(E1)}=[\mathrm{Rep}(A^E,R);\mathrm{Rep}(-A^E,R)]
    \). Denote $y^E\in\{-1,1\}^{2Rm^E}$ as the result vector of $A^{(E1)}$. Partition $y^E=[y_1^1;y_2^1;...;y_{m^E}^1;y_1^2;y_2^2;...;y_{m^E}^2]$ where each $y_i^{\{1,2\}}\in \{-1,1\}^R$. Let $y^0:=[y_1^0;y_2^0;...;y_{m^E}^0]\in\{0,1\}^{m^E}$ s.t. $y_i^0=0$ if $y_i^1=y_i^2=1^R$, otherwise $y_i^0=1$.
    \State Apply \cite[Algorithm~5]{li2026support} to $(A^E,y^0)$ and obtain the union support
$\beta=\{b_1,\ldots,b_{|\beta|}\}$. \begin{scriptsize}
        
    \end{scriptsize} 
    \State Let $A^{(E2)}:=[\mathrm{Rep}(M_{n,\beta}^S,R');\mathrm{Rep}(-M_{n,\beta}^S,R')]$. Denote $y^{(E2)}$ as the result vector of $A^{(E2)}$. Partition $y^{(E2)}:=[y_1^1;y_2^1;...;y_{|\beta|}^1;y_1^2;y_2^2;...;y_{|\beta|}^2]$, where $y_i\in\{-1,1\}^{R'}$. Initialize $S=\varnothing$. 
    \For {$i=1,2,...,|\beta|$}
    \begin{scriptsize}
        \Comment{In the second stage, we find all singletons}  
    \end{scriptsize}
    \State If $y_i^1$ and $y_i^2$ contain fewer than $\frac{1.5R'}{l}$ entries equal to $-1$ in total, then $S\leftarrow S\cup\{b_i\}$. 
    \EndFor
    
    \State Randomly initialize a feasible clustering of all $s\in S$ in the family of sets $\mc{B}:=\{B_1,B_2,...,B_l\}$.  
    \For{$j=1,2,...,\lceil\log\log(kl)\rceil$}\begin{scriptsize}\Comment{In the next stages, we find the distribution of singletons in candidate vectors}\end{scriptsize}
    \State Using $y^{Q^{j-1}},...,y^{Q^1}$ (defined below), generate $M^{Q^j}\in\{0,1\}^{m^{Q^j}\times n}$ as the matrix representation of the $j$-th stage pairwise queries from \cite[Theorem~1.1]{black2025learning}, applied to pairs of $S$. 
    \State Denote $y^{Q^j}_0\in\{-1,1\}^{2m^{Q^j}\cdot R'}$ as the result vector of \([\mathrm{Rep}(M^{Q^j},R');\mathrm{Rep}(-M^{Q^j},R')].\)
    \State Partition
    \(
    y^{Q^j}_0=[z_1^1;\ldots;z_{m^{Q^j}}^1;z_1^2;\ldots;z_{m^{Q^j}}^2],
    \)
    s.t. $z_i^{\{1,2\}}\in\{-1,1\}^{R'}$. Initialize $y^{Q^j}= []$.
    \For{$i=1,\ldots,m^{Q^j}$}
        \State If the $-1$ entries in $z_i^1$ $\&$ $z_i^2$ is $<\frac{1.5R'}{l}$, then append $1$ to $y^{Q^j}$. Else append $0$ to $y^{Q^j}$.
    \EndFor
\State Use $y^{Q^1},\ldots,y^{Q^j}$ and follow \cite[Theorem~1.1]{black2025learning} to update the clustering of all $s\in S$ in $\mathcal{B}$.
    \EndFor
    \State Let $B:=\{\beta_1,\beta_2,..,\beta_l\}$, where $\beta_i:=\mathrm{ind2vec}(B_i)\in\{0,1\}^n$, $i\in[l]$.
    \State Arbitrarily pick one element $\eta_i$ from every $B_i,i\in[l]$, and denote $\mc{N}:=\{\eta_1,...,\eta_l\}$. Order the elements in the set $\beta-S=\{t_1,t_2,...,t_{|\beta-S|}\}$. Generate the matrix $M^{(L)}$ as the matrix representation of measurements of pairs $(\eta_i,t_j)$ for all $i\in[l]$ and $t_j\in\beta-S$. % and let $M^{(L')}$ be its 3-fold magnification. 
    \State Denote $y^{(L)}$ as the result vector of $[\mathrm{Rep}(M^{(L)},R');\mathrm{Rep}(-M^{(L)},R')]$, and partition $y^{(L)}:=[y_{(1,1)}^1;y_{(1,2)}^1;...;y_{(l,|\beta-S|)}^1;y_{(1,1)}^2;y_{(1,2)}^2;...;y_{(l,|\beta-S|)}^2]$, where each $y_{(i,j)}^{\{1,2\}}\in \{0,1\}^{R'}$. 
    
    \For{$i\in[l]$, $j\in [|\beta-S|]$}
    \begin{scriptsize}
        \Comment{In the last stage, we check the frequency of all other pairs}
    \end{scriptsize}
    \State Let $F$ be s.t. the number of $-1$ entries of $y_j^1$ $\&$ $y_j^2$ in total is between $\frac{(F-0.5)R'}{l}$ and $\frac{(F+0.5)R'}{l}$. 
    \State If $y_{(i,j)}^1$ $\&$ $y_{(i,j)}^2$ contain fewer than $\frac{(F+0.5)R'}{l}$ entries equal to $-1$ in total, then $\beta_i\leftarrow\beta_i+e_{t_j}$.
    \EndFor
    \State \Return $B$ 

\end{algorithmic}
\end{algorithm}

%% file: 6_nonadaptive.tex
\section{MLC: Non-adaptive scheme}\label{sec:1SMLC}
% We now present our non-adaptive scheme. Recall that the two-stage scheme first recovers the union support $\beta$ and then uses frequency statistics of single coordinates (singletons) and coordinate pairs (doubletons) within $\beta$ to recover the candidate vectors. Without adaptivity, we require a measurement matrix $A\in\mathbb{R}^{m_A\times n}$ that enables these frequency tests for any possible union support $\beta$. In particular, for every set of at most $kl$ columns
% \(
% \mc A=\{A_{j_1},\ldots,A_{j_{kl}}\},
% \)
% to distinguish singletons, the standard disjunct property is required: for every column $A'\in\mc A$, there exists a coordinate at which $A'$ has a nonzero entry but all other columns in $\mc A$ has zero entries. To distinguish doubletons, we require that for every pairs of columns $A_1,A_2\in\mc A$, there exists a coordinate at which $A_1$ and $A_2$ have nonzero entries but all other columns in $\mc A$ has zero entries.  

% Such matrix $A$ can be constructed by concatenating $(d,r;z]$-disjunct matrices from~\cite{chen2008upper} with different parameter choices. The definition of this matrix is recalled below, also with the main theorem regarding the dimensions of this matrix:

We now present our non-adaptive scheme. Recall that the two-stage scheme first recovers the union support $\beta$ and then uses frequency statistics of single coordinates (singletons) and coordinate pairs (doubletons) within $\beta$ to recover the candidate vectors. Without adaptivity, we require a measurement matrix $A\in\mathbb{R}^{m_A\times n}$ that supports these frequency tests for every possible union support $\beta$. In particular, for every set of at most $kl$ columns

$$
\mc A=\{A_{j_1},\ldots,A_{j_{kl}}\},
$$
of $A$, we require the following properties: To distinguish singletons, for every column $A'\in\mc A$, there must exist a coordinate at which $A'$ has a nonzero entry while all other columns in $\mc A$ have zero entries. To distinguish doubletons, for every pair of columns $A_1,A_2\in\mc A$, there must exist a coordinate at which both $A_1$ and $A_2$ have nonzero entries while all other columns in $\mc A$ have zero entries.

Such a matrix $A$ can be constructed by concatenating $(d,r;z]$-disjunct matrices from~\cite{chen2008upper} with different parameter choices. We recall the definition of these matrices below, along with the main theorem giving their dimensions.

\begin{definition}
    \cite{chen2008upper} A matrix $M$ is called $(d,r;z]$-disjunct if for any $d + r$ columns $C_1, C_2,..., C_{d+r}$ of $M$, we have
    \[\left|\bigcap_{i=1}^r C_i\backslash \bigcup_{i=r+1}^{d+r}C_i\right|\geq z. \]
    That is, for any $d + r$ columns there exist at least $z$ rows where each of the $r$
    designated columns has nonzero entries and each of the other $d$ columns has zero entries.
\end{definition}
\begin{theorem}\label{thm:drz}
    \cite[Theorem~3.2]{chen2008upper} Let $t(n,d,r;z]$ denote the minimum number of rows among all $(d, r; z]$-disjunct
matrices with n columns. For any positive integers $d, r, z$ and $n$ with $d + r \leq n$, we have 
\[t (n, d, r; z] < z\left(\frac{k}{r}\right)^r \left(\frac{k}{d}\right)^d \left[1 + k\left(1 + \ln\left(\frac{n}{k+1}\right)\right)\right].\]
\end{theorem}

% In our case, distinguishing singletons requires a $(kl-1,1,R']$-disjunct matrix, and distinguishing doubletons requires a $(kl-2,2,R']$-disjunct matrix, where $R'$ is given by Lemma~\ref{lem:rep2}. To acheive out goal, the matrix $A$ can be obtained simply by concatenating  a $(kl-1,1,R']$-disjunct matrix and a $(kl-2,2,R']$-disjunct matrix. By Theorem~\ref{thm:drz}, $A$ has 

% \[
% O\!\left(R'(kl)^3\log\frac{n}{kl}\right)
% =
% O\!\left(k^3l^5\log(n)\log(kl\lambda)\right)
% \]
% rows.
% %\sout{where $\epsilon=0.001$}

% The overall measurement matrix will be obtained by concatenating
% \(
% A^{(1)},-A^{(1)},A,\text{ and }-A,
% \)
% where $A^{(1)}$ is the first-stage matrix from Algorithm~\ref{alg:2SUMLC} used to recover the union support. The decoding procedure follows the same frequency-based strategy as the two-stage scheme, except that all singleton and pair queries are now performed simultaneously. The complete construction and decoding algorithms are presented in Algorithms~\ref{alg:NA} and~\ref{alg:NAdec}.

In our setting, distinguishing singletons requires a $(kl-1,1,R']$-disjunct matrix, while distinguishing doubletons requires a $(kl-2,2,R']$-disjunct matrix, where $R'$ is given by Lemma~\ref{lem:rep2}. Thus, we can construct $A$ by concatenating a $(kl-1,1,R']$-disjunct matrix and a $(kl-2,2,R']$-disjunct matrix. By Theorem~\ref{thm:drz}, the resulting matrix $A$ has

$$
O\!\left(R'(kl)^3\log\frac{n}{kl}\right)
=
O\!\left(k^3l^5\log(n)\log(kl\lambda)\right)
$$
rows.

The overall measurement matrix is obtained by concatenating

$$
A^{(1)},-A^{(1)},A,-A,
$$
where $A^{(1)}$ is the first-stage matrix from Algorithm~\ref{alg:2SUMLC} used to recover the union support. The decoding procedure follows the same frequency-based strategy as in the two-stage scheme, except that all singleton and pair queries are performed simultaneously. The complete construction and decoding procedures are given in Algorithms~\ref{alg:NA} and~\ref{alg:NAdec}. 

\begin{algorithm}
% \scriptsize
\caption{Non-adaptive scheme of the MLC under Assumption~\ref{Assumption1}: Measurement matrix}\label{alg:NA}
\begin{algorithmic}[1]
    \Require $n$, $l$, $A$, $R$, and $R'$ from Algorithm~\ref{alg:2SUMLC}.
    \State Let $A^{(1)}:=\mathrm{Rep}(A,R)$. Let $A^{(D_1)}$ be a binary $(kl-1,1,R']$-disjunct matrix \cite{chen2008upper} and $A^{(D_2)}$ be a binary $(kl-2,2,R']$-disjunct matrix.
    
    % By \cite[Theorem 3.2]{chen2008upper}, the matrix $A_d:=[A^{(D_1)};A^{(D_2)}]\in\{0,1\}^{m_A\times n}$ exists with $m_A=O((kl)^3R'\log n)=O(k^3l^{5+\epsilon}\log (n) \log (kl\lambda))$.
    \State \Return $A_{na}:=[A^{(1)};-A^{(1)};A_d;-A_d]$.

\end{algorithmic}
\end{algorithm}

\begin{algorithm}
% \scriptsize
\caption{Non-adaptive scheme of the MLC under Assumption~\ref{Assumption1}: Decoding algorithm}\label{alg:NAdec}
\begin{algorithmic}[1]
    \Require $n,l,m',R,R',m_A,A',$ from Algorithm~\ref{alg:NA}. The test result vector $y$ from $A_{na}=[A^{(1)};-A^{(1)};A^{(D_1)};A^{(D_2)};-A^{(D_1)};-A^{(D_2)}]$ in Algorithm~\ref{alg:NA}, and parameter $\lambda>0$.   
    \State Partition $y=[y^{(1)};y^{(2)}]$, where $y^{(1)}\in\{-1,1\}^{2m'R}$,  $y^{(2)}\in\{-1,1\}^{2m_A}$. Further partition $y^{(1)}=[y_1^1;y_2^1;...;y_{m'}^1;y_1^2;y_2^2;...;y_{m'}^2]$ where $y_i\in \{-1,1\}^R$. Let $y^0:=[y_1^0;y_2^0;...;y_{m'}^0]\in\{0,1\}^{m'}$ s.t. $y_i^0=0$ if $y_i^1=y_i^2=1^R$, otherwise $y_i^0=1$. 
    \State Apply \cite[Algorithm~3]{li2026support} to $(A',y^0)$ and obtain the union support $\beta$. Let $A^{(D_1)}_\beta$ be the submatrix with column index set $\beta$ (see Basic Notations). Let $A^{(D_2)}_\beta$ be defined similarly. 
    \State By the ($kl-1,1,R'$] and ($kl-2,2,R'$]-disjunct properties of $A^{(D_1)}$ and $A^{(D_2)}$, respectively, there exists a submatrix of $[A^{(D_1)}_\beta;A^{(D_2)}_\beta;-A^{(D_1)}_\beta;-A^{(D_2)}_\beta]$ that equals $[\mathrm{Rep}(M_{n,\beta},R')_\beta;\mathrm{Rep}(-M_{n,\beta},R')_\beta]$, with changes of some non-zero values and some row permutation. Let $y^{(2)}_{perm}$ be the result vector corresponding to this submatrix. Partition $y^{(2)}_{perm}:=[y_1^1;y_2^1;...;y_{|\beta|}^1;y_{(1,2)}^1;y_{(1,3)}^1;...;y_{(|\beta|-1,|\beta|)}^1;y_1^2;y_2^2;...;y_{|\beta|}^2;y_{(1,2)}^2;y_{(1,3)}^2;...;y_{(|\beta|-1,|\beta|)}^2]$, where $y_i^{\{1,2\}},y_{(i,j)}^{\{1,2\}}\in \{-1,1\}^{R'}$, and initialize $S=\varnothing$ to collect singletons. 
    \For {$i=1,2,...,|\beta|$}
    \begin{scriptsize}
        \Comment{First, we find all singletons}
    \end{scriptsize}
    \State If $y_i^1$ and $y_i^2$ contain fewer than $\frac{1.5R'}{l}$ entries equal to $-1$ in total, then $S\leftarrow S\cup\{i\}$.
    \EndFor
    \State Initialize a set $B$ of $|S|$ vectors $\beta_s:=e_s\in\{0,1\}^n$ for $s\in S$, where $e_s$ is the $s$-th unit vector. 
    \For {each pair $\{\beta_i,\beta_j\}\subseteq B$, $i<j$}
    \begin{scriptsize}
        \Comment{Then, we check all pairs of singletons and merge them if needed}
    \end{scriptsize}
    \State If $y_{(i,j)}^1$ $\&$ $y_{(i,j)}^2$ contain $<\frac{1.5R'}{l}$ entries equal to $-1$ in total, then $\beta_i\leftarrow\beta_i+\beta_j$, $B\leftarrow B-\beta_j$. 
    \EndFor
    \For {each pair $(i,j)$ s.t. $i\in S$, $\beta_i\in B$ and $j\in \beta-S$}
    \begin{scriptsize}
        \Comment{Finally, we check all $(s,t)$ s.t. $s\in S$,$t\in\beta-S$.}
    \end{scriptsize}
    \State Let $F$ be s.t. the number of $-1$ entries of $y_j^1$ $\&$ $y_j^2$ in total is between $\frac{(F-0.5)R'}{l}$ and $\frac{(F+0.5)R'}{l}$. 
    \State If $y_{(i,j)}^1$ $\&$ $y_{(i,j)}^2$ contain fewer than $\frac{(F+0.5)R'}{l}$ entries equal to $-1$ in total, then $\beta_i\leftarrow\beta_i+e_j$.
    \EndFor
    \State \Return $B$ 
\end{algorithmic}
\end{algorithm}

We have the following main theorem for this section: 

\begin{theorem}\label{thm3}
    Let $B:=\{\beta_1,\beta_2,...,\beta_l\}$ be a set of unknown $k$-sparse vectors in $\mathbb{R}^n$ that satisfy Assumption~\ref{Assumption1}. Let 
    % \sout{$\epsilon=0.001$ and }
    $\lambda>0$,  there exists a non-adaptive algorithm (Algorithm~\ref{alg:NA},~\ref{alg:NAdec}) to recover the support of every candidate vector $\beta_i\in B$ with probability at least $1 - \frac{2}{\lambda}$, using 
    % \sout{$m=O(k^3l^{5+\epsilon}\log^2(n)\log(kl\lambda\log n))$} 
    $m=O(k^3l^5\log^2(n)\log(kl\lambda\log n))$ queries of the oracle $\mc{O}$ with a decoding complexity $O(klm)$. 
\end{theorem}

%% file: 7_conclusion.tex
\section{Future work}\label{sec:conclusion}
% \xl{Several directions for future research remain open. While the present work focuses on achievability results, it is of interest to establish corresponding lower bounds, particularly to better understand the potential computational gap between adaptive and non-adaptive schemes. Another important direction is to relax the structural assumptions on the supports of the candidate vectors; for instance, one may consider settings in which the supports are only required to be distinct \cite{pal2021support,polyanskii2021learning}. Finally, extending the proposed framework to noisy settings is a natural and practically relevant avenue for further study. In particular, it would be of interest to analyze performance under the bounded Massart noise model~\cite{awasthi2016learning}, where each measurement outcome is independently flipped with probability $\eta < 1/2$.}

Several directions remain open for future work. First, while we focus on achievability, establishing corresponding lower bounds would clarify the fundamental limits and potential computational gap between adaptive and non-adaptive schemes. Second, it would be interesting to relax the support assumptions, such as considering settings where candidate supports are only required to be distinct~\cite{pal2021support,polyanskii2021learning}. Finally, extending our framework to noisy measurements, particularly under bounded Massart noise~\cite{awasthi2016learning}, is a natural direction.

%% file: Appendices.tex
\section{Appendices}
\subsection{Proof of Lemma~\ref{lem:rep1}}
% \begin{lemma}
%     With $R:=\lceil l\cdot\log(m'l\lambda)\rceil$ repetitions of each measurement, we can guarantee that $\sgn(\langle v_j, \beta_i\rangle)$ appear somewhere in the result vector for all $i\in[l]$ and $j\in[m']$, with failure probability $<\frac{1}{\lambda}$ for any $\lambda>0$. 
% \end{lemma}
\begin{proof}
    Let $r > 0$ and define $R := l \cdot r$. Consider a fixed candidate vector $\beta_1$. The probability that $R$ independent samples all miss $\beta_1$ is
\[
\left(1 - \frac{1}{l}\right)^R \leq e^{-R/l} = e^{-r}.
\]

By the union bound, the probability that at least one of the $l$ candidate vectors is missed after $R$ samples is at most
\[
l \cdot e^{-r}.
\]

\medskip

Since we will repeat all $m'$ measurements and their negations. Applying the union bound, the overall failure probability is at most
\[
2\cdot m' l \cdot e^{-r}.
\]

\medskip

To ensure that the failure probability is at most $\frac{1}{\lambda}$ for some $\lambda > 0$, it suffices to require
\[
2m' l \cdot e^{-r} \leq \frac{1}{\lambda}.
\]

Taking logarithms, this condition is equivalent to
\[
r \geq \log(2m' l \lambda).
\]

Therefore, recalling that $R = l r$, it suffices to take
\[
R \geq l \cdot \log(2m' l \lambda).
\]
\end{proof}

\subsection{Proof of Lemma~\ref{lem:trueb}}
\begin{proof}
%     \xl{
%     We refer to Algorithm~2 (measurement matrix construction) and Algorithm~3 (decoding algorithm) in~\cite{li2026support}. All italicized terms below follow the same meaning as in~\cite{li2026support}, and we adopt the same notation. In particular, write $A = [A'; A'']$.

% It suffices to show that every $b \in \beta$ is correctly identified, and that no \emph{false positives} are retained.

% \medskip

% \emph{First stage.}
% The goal of the first stage is to identify all \emph{singletons}. By the \emph{$(kl,1)$-distinguishable} property of the matrix $M$ used in Algorithm~2 (recall that $A'$ is the $U$-magnified version of $M$), for every $b \in \beta$, there exists a row $v$ of $M$ such that
% \[
% \supp(v) \cap \beta = \{b\}.
% \]

% Let $\beta_i \in B$ be any candidate vector with $b \in \supp(\beta_i)$ (such a vector must exist). Then 
% \[v \cdot \beta_i \neq 0 \Rightarrow \sign\langle v,\beta_i\rangle \text{ or }\sign\langle -v,\beta_i\rangle \text{ is } -1, \Rightarrow y^0_j=1 \text{ at correspoding positions }j.\]
% Combined with the singleton (the singleton here means the same as in \cite{li2026support}, not the singleton in our context) property of the coordinate $b$, we derive that $b$ is detected in the first stage. Therefore, all elements of $\beta$ are recovered after the first stage, possibly along with some false positives.
% }

We refer to Algorithms~2 and~3 of~\cite{li2026support}, which describe the construction of the measurement matrix and the corresponding decoding algorithm of the EDOCS-AE scheme, respectively. Throughout the proof, all italicized terms are used with the same meaning as in~\cite{li2026support}, and we adopt the notation therein. In particular, write $A=[A';A'']$.

It suffices to prove that every coordinate in $\beta$ is identified and that no \emph{false positives} remain after decoding.

\emph{First stage.}
The objective of the first stage is to identify all \emph{singletons} (in the sense of~\cite{li2026support}). Since the matrix $M$ used in~\cite[Algorithm~2]{li2026support} is $(kl,1)$-distinguishable~\cite[Definition~2]{li2026support}, and $A'$ is the $U$-magnified version of $M$~\cite[Definition~4]{li2026support}, for every $b\in\beta$ there exists a row $v$ of $M$ satisfying
\[
\supp(v)\cap\beta=\{b\}.
\]

Choose any candidate vector $\beta_i\in B$ with $b\in\supp(\beta_i)$, which exists by the definition of $\beta$. Then
\[
\langle v,\beta_i\rangle\neq0,
\]
and hence at least one of
\[
\sign(\langle v,\beta_i\rangle)
\quad\text{and}\quad
\sign(\langle -v,\beta_i\rangle)
\]
equals $-1$. By the definition of the tentative result vector in Section~\ref{sec:2SUMLC}, the corresponding entry of $y^0$ is therefore equal to $1$. Since $b$ is a singleton (in the sense of~\cite{li2026support}), the decoding algorithm identifies $b$ during the first stage. As this argument applies to every $b\in\beta$, all coordinates in $\beta$ are recovered after the first stage, possibly together with some false positives.

\medskip

\emph{Second stage.}
It remains to show that all true elements are retained and all false positives are removed.

Let
\[
\mathcal{C} := \bigcup_{b \in \beta} \supp(A''_b)
\]
denote the union of supports of columns of $A''$ indexed by $\beta$. Ideally, the result vector $y''$ would satisfy $\supp(y'') = \mathcal{C}$. However, due to possible accidental zeros, we only have $\supp(y'') \subseteq \mathcal{C}$.

For each $s \in [n]$, define
\[
\mathcal{C}_{-s} := \bigcup_{\substack{b \in \beta \\ b \neq s}} \supp(A''_b).
\]

\medskip

Now consider $b \in \beta$. By the $(n,m'',d'',kl,1,1/2)$ property of $A''$, where $m'' = O(k^2 l^2 \log n)$ and $d'' = O(kl \log n)$, we have
\[
|\supp(A''_b) \setminus \mathcal{C}_{-b}| \geq \frac{d''}{2}.
\]

All elements in $\supp(A''_b) \setminus \mathcal{C}_{-b}$ correspond to singletons and therefore must appear in $y''$. Hence,
\[
|\supp(y'') \cap \supp(A''_b)| \geq \frac{d''}{2},
\]
and $b$ is retained by the decoding rule.

\medskip

Next, consider $s \notin \beta$ (i.e. a false positive). Again by the same property of $A''$, we have
\[
|\mathcal{C} \cap \supp(A''_s)| < \frac{d''}{2}.
\]

Since $\supp(y'') \subseteq \mathcal{C}$, it follows that
\[
|\supp(y'') \cap \supp(A''_s)| < \frac{d''}{2},
\]
and hence $s$ is removed.
\end{proof}

\subsection{Proof of Lemma~\ref{lem:rep2}}
\begin{proof}
    % First note that we do not need to consider the case when $f^{(i)}=0$ or $f^{(i)}=l$, as in these cases, $X^{(i)}=0$ and $X^{(i)}=R'$ with probability one, respectively. 
    
    % By the Multiplicative Chernoff bound, let $x:=l-0.5$, and using the fact that $((1-\frac{0.5}{x})^x\cdot e^{0.5})^{x^{1+\epsilon}}<0.9$ for $\epsilon=0.001$ and $x\geq0.5$, we have
    % \[P\left(X^{(i)}>\frac{(f^{(i)}+0.5)R'}{l}\right)=P\left(X^{(i)}>\frac{f^{(i)}+0.5}{f^{(i)}}\cdot E[X^{(i)}]\right)\leq\left(\frac{e^{\frac{0.5}{f^{(i)}}}}{(\frac{f^{(i)}+0.5}{f^{(i)}})^{\frac{f^{(i)}+0.5}{f^{(i)}}}}\right)^{\frac{f^{(i)}R'}{l}}\]
    % \[\leq \left(\frac{e^{0.5}}{(1+\frac{0.5}{l-1})^{l-0.5}}\right)^\frac{R'}{l}<\left(\left(1-\frac{0.5}{x}\right)^x\cdot e^{0.5}\right)^{x^{1+\epsilon}\cdot r'}<0.9^{r'}.\]
    % Similarly, using the fact that $((1+\frac{0.5}{x})^x\cdot e^{-0.5})^{x^{1+\epsilon}}<0.93$ for $\epsilon=0.001$ and $x\geq0.5$, we have
    % \[
    % P\left(X^{(i)}<\frac{(f^{(i)}-0.5)R'}{l}\right)=P\left(X^{(i)}<\frac{f^{(i)}-0.5}{f^{(i)}}\cdot E[X^{(i)}]\right)\leq\left(\frac{e^{-\frac{0.5}{f^{(i)}}}}{(\frac{f^{(i)}-0.5}{f^{(i)}})^{\frac{f^{(i)}-0.5}{f^{(i)}}}}\right)^{\frac{f^{(i)}R'}{l}}\]
    % \[\leq \left(\frac{e^{-0.5}}{(1-\frac{0.5}{l})^{l-0.5}}\right)^\frac{R'}{l}<\left(\left(1+\frac{0.5}{x}\right)^{x}\cdot e^{-0.5}\right)^{x^{1+\epsilon}\cdot r'}<0.93^{r'}.
    % \]
    
    % By union bound, the overall failure probability is then $\leq 2\cdot 2|\beta|^2\cdot 0.93^{r'}\leq 4k^2l^2\cdot 0.93^{r'}$. For this probability to be smaller than $\frac{1}{\lambda}$, it suffices that $r'\geq 14\log(4k^2l^2\lambda)$. 

    Note that 
    \(E[X^{(i)}]=R'\cdot \frac{f^{(i)}}{l}\) and $ X_j^{(i)}\in\{0,1\}$.    
    Let $t:=\frac{0.5R'}{l}$, by Hoeffding's inequality, we have 
    \[P\left(\frac{(f^{(i)}-0.5)R'}{l}<X^{(i)}<\frac{(f^{(i)}+0.5)R'}{l}\text{ fails}\right)\]
    \[=P(|X^{(i)}-E[X^{(i)}]|\geq t)\leq 2\exp(-\frac{2t^2}{R'})=2\exp(-\frac{0.5R'}{l^2})<2\exp(-0.5r'). \]
    By union bound, the overall failure probability is then upper bounded by \[2\cdot 2|\beta|^2\cdot \exp(-0.5r')\]
    \[\leq 4k^2l^2\cdot \exp(-0.5r').\] 
    For this quantity to be smaller than $\frac{1}{\lambda}$, it suffices that \[r'\geq 2\log(4k^2l^2\lambda).\]

\end{proof}

\subsection{Proof of Theorem~\ref{thm1}}
\begin{proof}
We construct the first- and second-stage measurement matrices, $A^{(1)}$ and $A^{(2)}$, using Algorithm~\ref{alg:2SUMLC}. The decoding algorithm consists of Algorithm~\ref{alg:2SUMLC} followed by Algorithm~\ref{alg:mlcdadec}. The proof follows from the preceding discussion, together with an analysis of the accidental zero issue.

\emph{Correctness.}
In the first stage, we recover the union support
\[
\beta:=\bigcup_{i\in[l]}\supp(\beta_i)
\]
using an $R$-fold repetition of the one-bit compressed sensing measurement matrix $A$ and its negation $-A$. Two types of errors may occur:
\begin{itemize}
\item[(i)] the global AC event fails; or
\item[(ii)] an accidental zero occurs.
\end{itemize}
By Lemma~\ref{lem:rep1}, the first event occurs with probability at most $1/\lambda$, while Lemma~\ref{lem:trueb} rules out the second. Hence, the first stage correctly recovers $\beta$.

We now analyze the second stage. Before proceeding, we verify that Lemma~\ref{lem:rep2} applies throughout this stage. The first part examines $|\beta|$ singleton coordinates, each associated with a measurement vector and its negation, contributing $2|\beta|$ random variables. The second part examines pairs of singleton coordinates, contributing fewer than
\[
2\binom{|\beta|}{2}
\]
random variables. The final part examines pairs consisting of one singleton and one non-singleton coordinate, contributing fewer than another
\[
2\binom{|\beta|}{2}
\]
random variables. Therefore, the total number of random variables involved in the second stage is at most
\[
2|\beta|+4\binom{|\beta|}{2}
=2|\beta|^2,
\]
which is precisely the regime covered by Lemma~\ref{lem:rep2}. Thus, Lemma~\ref{lem:rep2} may be invoked throughout the remainder of the proof.

In the first part of the second stage, we identify the singleton coordinates in $\beta$ by examining their frequencies. For example, suppose coordinate $1$ is a singleton. Algorithm~\ref{alg:mlcdadec} repeats the measurement vector
\[
[1,0,0,\ldots,0]\in\{0,1\}^n
\]
and its negation
\[
[-1,0,0,\ldots,0]
\]
exactly $R'$ times each. Since coordinate $1$ belongs to exactly one candidate vector, the expected number of measurements equal to $-1$ is $\frac{R'}{l}$. Therefore, Lemma~\ref{lem:rep2} implies that the threshold $\frac{1.5R'}{l}$ correctly identifies the singleton set $S$.

In the second part, we examine every pair of singleton coordinates and merge them whenever they belong to the same candidate vector. Without loss of generality, suppose $1,2\in S$. Algorithm~\ref{alg:mlcdadec} repeats the measurement vectors
\[
[1,1,0,\ldots,0]
\quad\text{and}\quad
[-1,-1,0,\ldots,0]
\]
exactly $R'$ times each. If coordinates $1$ and $2$ belong to the same candidate vector, then an accidental zero causes every measurement corresponding to that vector to evaluate to $+1$, so the number of measurements equal to $-1$ is exactly zero; otherwise, its expectation is $\frac{R'}{l}$. If the two coordinates belong to different candidate vectors, then the expected number of measurements equal to $-1$ is $\frac{2R'}{l}$, regardless of the occurrence of an accidental zero. Hence, Lemma~\ref{lem:rep2} shows that the threshold $\frac{1.5R'}{l}$ distinguishes the two cases with high probability.

Finally, we examine every pair $(s,t)$ with $s\in S$ and $t\in\beta\setminus S$. Without loss of generality, suppose $1\in S$ and $2\in\beta\setminus S$. Let $F$ denote the quantity defined in line 14 of Algorithm~\ref{alg:mlcdadec}. By Lemma~\ref{lem:rep2},
\[
F=\mathrm{Freq}(2,B)
\]
with high probability. The same two cases as above apply. If coordinates $1$ and $2$ belong to the same candidate vector, then an accidental zero reduces the expected number of measurements equal to $-1$ from $\frac{FR'}{l}$ to $\frac{(F-1)R'}{l}$. Otherwise, the expected number of measurements equal to $-1$ is $\frac{(F+1)R'}{l}$, regardless of the occurrence of an accidental zero. Therefore, Lemma~\ref{lem:rep2} implies that the threshold
\[
\frac{(F+0.5)R'}{l}
\]
distinguishes the two cases with high probability.

At this point, we established the correctness of this scheme.

\emph{Number of measurements.}
The first stage uses
\[
m_1 := 2\cdot m' \cdot R = O\!\left(k^2 l^2 \log\!\left(\frac{n}{kl}\right)\log (n) \cdot \lceil l \log(m' l \lambda)\rceil\right)
= O\!\left(k^2 l^3 \log^2 (n) \log(kl \lambda \log n)\right)
\]
measurements.

The second stage uses at most
% \[
% 2|\beta|^2 \cdot 2R'
% = O\!\left(k^2 l^2 \cdot l^{2+\epsilon} \log(k^2 l^2 \lambda)\right)
% \]

\[
m_2:=2|\beta|^2 \cdot 2R'
= O\!\left(k^2 l^2 \cdot l^2 \log(k^2 l^2 \lambda)\right)=O\left(k^2l^4\log(kl\lambda)\right)
\]
measurements.  
% which is asymptotically smaller than $m_1$. 
Therefore, the total number of measurements is
% \[
% m = O\!\left(k^2 l^{4+\epsilon} \log^2 (n) \log(kl \lambda \log n)\right).
% \]

\[
m = m_1+m_2=O\!\left(k^2 l^4 \log^2 (n) \log(kl \lambda \log n)\right).
\]

\medskip

\emph{Decoding complexity.}
In the first stage, the decoder scans the result vector $y \in \{-1,1\}^{2m_1}$ and aggregates it into a vector $y^0 \in \{0,1\}^{m'}$. It then applies the EDOCS-AE decoding algorithm, which runs in time $O(kl m')$. Thus, the total complexity of the first stage is
\[
O\bigl(\max(m_1, kl m')\bigr) = O\!\left(k^3 l^3 \log^2 (n)\log(kl\lambda\log n)\right).
\]

In the second stage, Algorithm~\ref{alg:mlcdadec} consists of three nested loops. In the worst case, the loop over all pairs $\{\beta_{s_1}, \beta_{s_2}\} \subseteq B$ requires 
% \sout{$O(|\beta|^2\cdot R') = O(k^2 l^{4+\epsilon}\log(kl\lambda))$} 
$O(|\beta|^2\cdot R') = O(k^2 l^4\log(kl\lambda))$ time, while the remaining loops are asymptotically smaller.

\medskip

Combining both stages, the overall decoding complexity is
% \[
% O\!\left(k^3 l^{4+\epsilon} \log^2 (n)\log(kl\lambda \log n)\right)=O(km).
% \]

\[
O\!\left(k^3 l^4\log^2 (n)\log(kl\lambda \log n)\right)=O(km).
\]

\end{proof}

\subsection{Proof of Theorem~\ref{thm2}}
\begin{proof}
    We construct the measurement matrices and perform decoding using Algorithm~\ref{alg:MSPMLC}. The notation follows accordingly.

\medskip

\emph{Correctness.}
The correctness of the scheme can be proved in a similar way as that of Theorem~\ref{thm1}, since their construction and decoding algorithm are essentially identical.

\medskip

\emph{Number of measurements.}
In the first stage, the number of measurements is
\[
m^E \cdot R
= O\!\left(kl \cdot \frac{\log(kl)}{\log\log(kl)} \cdot \log n\right) \cdot R
= O\!\left(kl^2 \log(kl)\log (n) \log(kl\lambda \log n)\right).
\]

In the second stage, the number of measurements is
% \[
% |\beta| \cdot R'
% = O\!\left(kl \cdot l^{2+\epsilon} \log(k^2 l^2 \lambda)\right)
% = O\!\left(kl^{3+\epsilon} \log(kl\lambda)\right).
% \]

\[
|\beta| \cdot R'
= O\!\left(kl \cdot l^2 \log(k^2 l^2 \lambda)\right)
= O\!\left(kl^3\log(kl\lambda)\right).
\]

For the subsequent $\lceil \log\log(kl) \rceil$ stages, by~\cite[Theorem~1.1]{black2025learning}, the total number of measurements is bounded by
% \[
% 8(kl)^{1+\frac{1}{2^{\lceil\log\log(kl)\rceil}-1}}
% \cdot l^{1-\frac{1}{2^{\lceil\log\log(kl)\rceil}-1}}
% \cdot 3R'
% = O\!\left(kl^2 \cdot R'\right)
% = O\!\left(kl^{4+\epsilon} \log(kl\lambda)\right).
% \]

\[
8(kl)^{1+\frac{1}{2^{\lceil\log\log(kl)\rceil}-1}}
\cdot l^{1-\frac{1}{2^{\lceil\log\log(kl)\rceil}-1}}
\cdot 3R'
= O\!\left(kl^2 \cdot R'\right)
= O\!\left(kl^4\log(kl\lambda)\right).
\]

In the final stage, the number of measurements is
% \[
% 3l|\beta - S| \cdot R'
% = O\!\left(l \cdot kl \cdot l^{2+\epsilon} \log(k^2 l^2 \lambda)\right)
% = O\!\left(kl^{4+\epsilon} \log(kl\lambda)\right).
% \]

\[
3l|\beta - S| \cdot R'
= O\!\left(l \cdot kl \cdot l^2 \log(k^2 l^2 \lambda)\right)
= O\!\left(kl^4 \log(kl\lambda)\right).
\]

Combining all stages, the total number of measurements is
% \[
% m = O\!\left(kl^{4+\epsilon} \log(kl)\log (n) \log(kl\lambda \log n)\right).
% \]

\[
m = O\!\left(kl^4\log(kl)\log (n) \log(kl\lambda \log n)\right).
\]

\emph{Decoding complexity.}
In the first stage, the decoder scans the result vector $y^E \in \{-1,1\}^{2R m^E}$ and aggregates it into $y^0 \in \{0,1\}^{m^E}$. It then applies the EDOCS-EE decoding algorithm~\cite[Algorithm~5]{li2026support}, which runs in time $O(m^E)$. Thus, the total complexity of the first stage is
\[
O(m^E \cdot R)
= O\!\left(kl^2 \log(kl)\log (n) \log(kl\lambda \log n)\right).
\]

For all subsequent stages, the decoding procedures are linear in the number of measurements. Therefore, their total complexity is asymptotically bounded by the total number of measurements.

\medskip

Hence, the overall decoding complexity is
$O(m)$.
\end{proof}

\subsection{Proof of Theorem~\ref{thm3}}
\begin{proof}
    We will use Algorithm~\ref{alg:NA} to construct our measurement matrices and Algorithm~\ref{alg:NAdec} to decode. The notations also follow. 

    The correctness of the scheme follows directly, as this scheme is essentially the same to that in Theorem~\ref{thm1}.

    In terms of the number of measurements, the matrix $A^{(1)}$ uses $m_1=O(k^2l^3\log^2(n)\log(kl\lambda\log n))$ measurements, and the matrix $A=[A^{(D_1)};A^{(D_2)}]$ uses % \sout{$O(k^3l^{5+\epsilon}\log(n)\log(kl\lambda))$} 
    $O(k^3l^5\log(n)\log(kl\lambda))$ measurements (Justified by Algorithm~\ref{alg:NA}). Therefore, the total number of measurement will be 
    %\[m=O(k^3l^{5+\epsilon}\log^2(n)\log(kl\lambda\log n)). \]
    \[m=O(k^3l^5\log^2(n)\log(kl\lambda\log n)). \]

    In terms of decoding complexity, the first part is the same as that of the first stage in Theorem~\ref{thm1}, i.e. 
    \[O(k^3l^3\log^2(n)\log(kl\lambda\log n)).\]

    In the second part, the decoder first finds a submatrix of $A_\beta$ that equals $\mathrm{Rep}(M_{n,\beta}',R')_\beta$. This process requires 
    % \[|\beta|\cdot m_A=O(kl\cdot k^3l^{5+\epsilon}\log(n)\log(kl\lambda))=O(k^4l^{6+\epsilon}\log(n)\log(kl\lambda))\]
    
    \[|\beta|\cdot m_A=O(kl\cdot k^3l^5\log(n)\log(kl\lambda))=O(k^4l^6\log(n)\log(kl\lambda))\]
    entries to be examined. 
    
    % \textcolor{blue}{A bit more may be done here. The column weight of the disjunct matrix in \cite{chen2008upper} should be bounded by a smaller value than the column length. 
    % \\ The disjunct matrix in \cite{chen2008upper} is constructed by constant-row-width (around $O(n/kl)$) vectors, so if this matrix is used, we can somehow only guarantee with high-probability that the decoding algorithm will have complexity only $O(m)$ instead of $O(klm)$. The logic is that the "density" of this matrix is $1/kl$. 
    % }

    Then, the decoder follows the process as in Algorithm~\ref{alg:mlcdadec}, with complexity only 
    %\sout{$O(k^2l^2\cdot R')=O(k^2l^{4+\epsilon}\log(kl\lambda))$} 
    \[O(k^2l^2\cdot R')=O(k^2l^4\log(kl\lambda)).\] 

    In summary, the overall decoding complexity will be 
    \(O(klm)\).
\end{proof}